\documentclass{article}

\PassOptionsToPackage{numbers,compress}{natbib}
\usepackage[preprint]{neurips_2023}
\usepackage[utf8]{inputenc}
\usepackage[T1]{fontenc}
\usepackage{url}
\usepackage{amsfonts}
\usepackage{microtype}
\usepackage{xcolor}

\usepackage{eccvabbrv}

\usepackage{graphicx}
\usepackage{booktabs}

\usepackage{wrapfig}
\usepackage{caption}

\usepackage[accsupp]{axessibility}  

\usepackage{hyperref}

\usepackage{orcidlink}

\usepackage{booktabs}
\usepackage{multirow}
\usepackage{tabularx}
\usepackage{placeins}
\usepackage{amsthm}
\usepackage{siunitx}
\newtheorem{theorem}{Theorem}

\begin{document}

\title{CausalVAE as a Plug-in for World Models: Towards Reliable Counterfactual Dynamics}

\author{
  Ziyi Ding \\
  Tsinghua Shenzhen International Graduate School, Tsinghua University \\
  \texttt{dingzy25@mails.tsinghua.edu.cn}
  \And
  Xianxin Lai \\
  The University of Hong Kong \\
  \texttt{xianxin.lai@connect.hku.hk}
  \And
  Weiyu Chen \\
  The University of Hong Kong \\
  \texttt{weiyuc@connect.hku.hk}
  \And
  Xiao-Ping Zhang \\
  Tsinghua Shenzhen International Graduate School, Tsinghua University \\
  \texttt{xpzhang@ieee.org}
  \And
  Jiayu Chen\thanks{Corresponding author.} \\
  The University of Hong Kong \\
  INFIFORCE Intelligent Technology \\
  \texttt{jiayuc@hku.hk}
}

\maketitle

\begin{abstract}
  In this work, CausalVAE is introduced as a plug-in structural module for latent world models and is attached to diverse encoder-transition backbones. Across the reported benchmarks, competitive factual prediction is preserved and intervention-aware counterfactual retrieval is improved after the plug-in is added, suggesting stronger robustness under distribution shift and interventions. The largest gains are observed on the Physics benchmark: when averaged over 8 paired baselines, CF-H@1 is improved by \textbf{+102.5\%}. In a representative GNN-NLL setting on Physics, CF-H@1 is increased from \textbf{11.0} to \textbf{41.0} (\textbf{+272.7\%}). Through causal analysis, learned structural dependencies are shown to recover meaningful first-order physical interaction trends, supporting the interpretability of the learned latent causal structure.
\end{abstract}

\section{Introduction}
Generalization under distribution shifts, interventions, and mechanism changes remains a central challenge for visual model-based learning \cite{schoelkopf2021towards}.
World models achieve strong predictive performance by compressing observations into latent states and rolling them forward under actions \cite{ha2018worldmodels,hafner2019dreamer,hafner2023dreamerv3}, but predictive latents are often entangled and weakly aligned with underlying causal factors, limiting out-of-distribution robustness.

Recent work on causal representation learning argues that discovering high-level causal variables from low-level sensory inputs is essential for robust generalization and transfer \cite{schoelkopf2021towards}.
In parallel, object-centric representation learning provides a compositional inductive bias by decomposing scenes into sets of object slots, improving systematic generalization \cite{locatello2020slotattention,greff2019iodine,burgess2019monet}.
However, neither standard world models nor generic object-centric models explicitly enforce that latent factors obey a directed acyclic graph (DAG) causal structure. This limitation is critical: without an explicit structural causal model, latent interventions are not identifiable, so counterfactual rollouts can deviate from physically valid alternative trajectories.

To address this, we integrate a structured causal disentanglement module into a world-model pipeline.
Our structural branch adopts the CausalVAE causal layer and mask mechanism to map independent exogenous factors into causally related endogenous variables, while learning a DAG over the latent factors \cite{yang2020causalvae}.
We enforce acyclicity via a differentiable DAG constraint inspired by continuous optimization approaches for structure learning \cite{zhengDAGsNOTEARS2018}.
Because our setting is sequential and action-conditioned, existing static causal representation methods (e.g., CausalVAE in i.i.d. settings) cannot be directly applied for stable multi-step dynamics. Therefore, we introduce a staged training strategy that first learns predictive dynamics and then progressively activates structural regularization, while using alignment-only weak supervision to anchor latent coordinates (instead of the auxiliary-variable conditional prior used in identifiable VAE analyses \cite{pmlr-v108-khemakhem20a}).
This design yields an interpretable causal state space that supports intervention-style reasoning and improves robustness in model-based prediction.

\textbf{Our contributions are}:
(i) a plug-in causal structural module for latent world models, which can be combined with diverse representative world-modeling baselines used in our experiments;
(ii) a latent world-model formulation that explicitly captures causal relationships among latent elements, rather than treating latent dimensions as purely predictive features;
(iii) a staged optimization scheme for stable sequential training, together with an alignment-anchored identifiability analysis that characterizes what becomes identifiable under this training setup, and empirical gains in generalization and counterfactual consistency over non-causal world-modelling baselines.

\section{Related Work}

\subsection{Structured World Models and Latent Dynamics.}
World models learn compact latent states to support multi-step prediction 
and visual planning \cite{ha2018worldmodels,hafner2020dreamerv2}.
 To improve compositionality in physical domains, structured variants 
 introduce object- or relation-centric inductive biases, 
 utilizing graph networks (e.g., NRI, GNS) 
 \cite{kipf2018nri,sanchezgonzalez2020gns}
  or object-centric visual slots (e.g., C-SWM, Slot Attention) 
  \cite{kipf2020cswm,locatello2020slotattention}. While these models,
   including standard GNN-based dynamics, achieve strong factual 
   rollout accuracy, their latent representations are optimized
    purely as predictive features. Without learning an explicit 
    Structural Causal Model (SCM) \cite{pearl2009causality}, the model does not capture the
    causal relationships among variables. As a result, it may perform
    well on factual prediction, but can break down under interventions
    or counterfactual tasks.

\subsection{Causal Representations and Counterfactual Dynamics.}
Causal representation learning grounds latents in SCMs to enable interpretable interventions. Methods such as CausalVAE \cite{yang2020causalvae} are mainly developed for static i.i.d.\ observations, where supervision and identifiability assumptions do not directly match action-conditioned sequential rollouts. In multi-step visual dynamics, this mismatch can manifest as compounding one-step errors and latent drift under temporal distribution shift. Our method therefore does not directly apply a static causal model: we introduce staged optimization and alignment-anchored supervision to stabilize sequential training while preserving a structured causal latent space. Conversely, recent counterfactual dynamics models (e.g., CWM, Causal-JEPA) \cite{li2020causalworldmodels,venkatesh2024cwm,nam2026causaljepa} emphasize learning and evaluation under counterfactual/interventional settings in visual environments, but many approaches rely on masking/prompting heuristics instead of explicitly integrating an intervention-ready causal factorization into the transition dynamics.

\subsection{Causal evaluation in dynamical systems.}
Prior work stresses that predictive accuracy is insufficient to evaluate causal correctness from pixels, requiring rigorous intervention-centric metrics \cite{ke2021systematic}. Counterfactual dynamics models further argue that counterfactual capability should be treated as a first-class objective beyond factual prediction \cite{li2020causalworldmodels,venkatesh2024cwm}. Motivated by these findings, we evaluate world models using intervention-centric protocols that probe interventional faithfulness and multi-step counterfactual consistency in addition to factual forecasting.

\begin{figure}[t]
    \centering
    \includegraphics[width=0.7\linewidth]{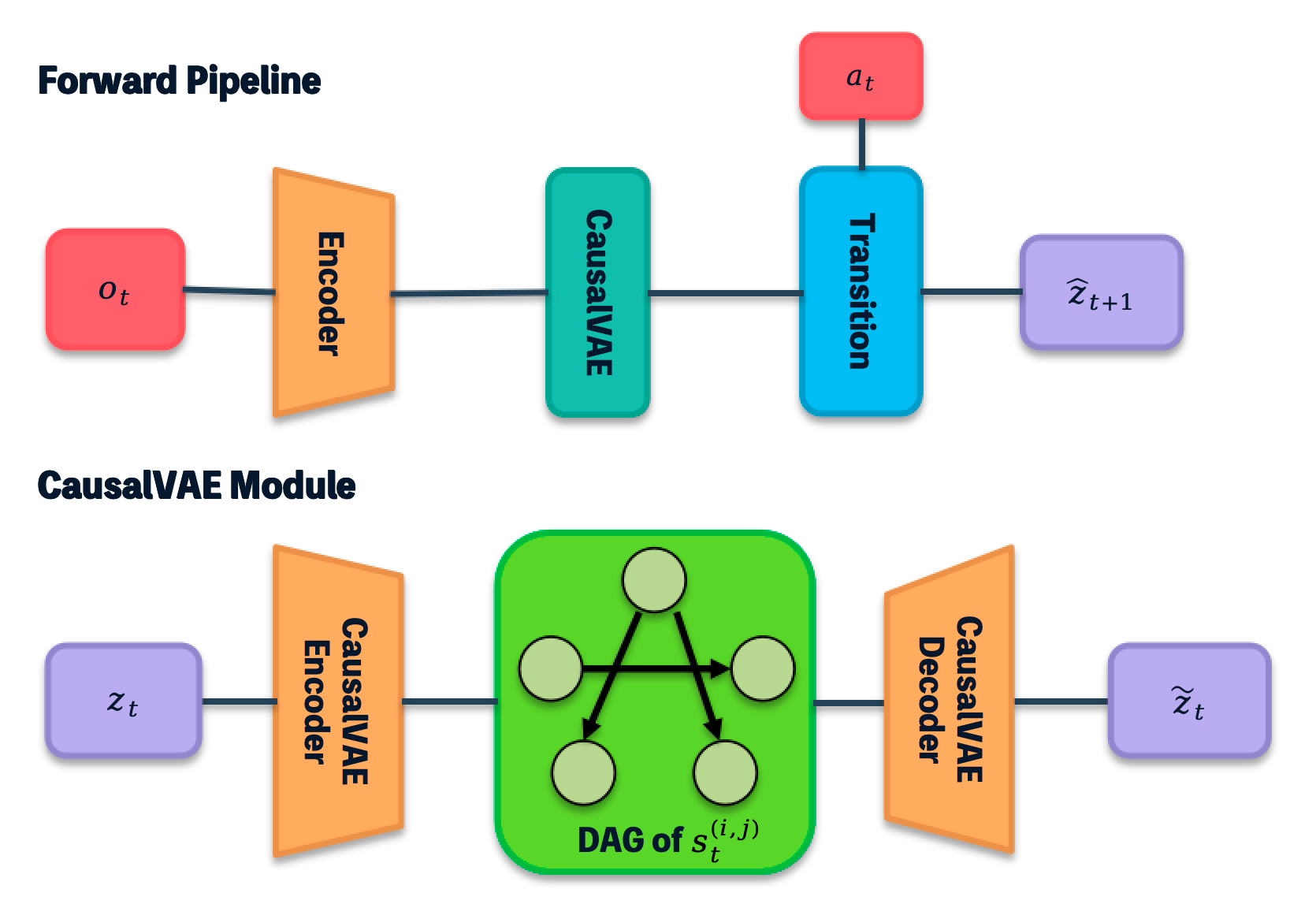}
    \caption{Overview of our framework. Top: $o_t \!\to\! z_t \!\to\! \tilde{z}_t \xrightarrow[]{a_t} \hat{z}_{t+1}$. Bottom: the CausalVAE branch imposes DAG-structured causal constraints in latent space before decoding back to $\tilde{z}_t$.}
    \label{fig:method_overview}
\end{figure}

\section{Method}
\subsection{Problem Setup}
\label{sec:problem_setup}

We study world modelling from visual observations under both factual and counterfactual settings. Following structured world-model formulations~\cite{kipf2020cswm,ke2021systematic}, we adopt an object-centric latent dynamics pipeline and augment it with a causal structural branch inspired by CausalVAE~\cite{yang2020causalvae}.

\noindent\textbf{Given.}
We are given transition tuples $\mathcal{D}=\{(o_t,a_t,o_{t+1})\}_{t=1}^{T}$, with optional simulator state $s_t\in\mathbb{R}^{d_s}$ used only for training-time alignment when available. Counterfactual queries are specified by intervention metadata $(i,j,\Delta)$ and implemented as $\mathrm{do}(s_t^{(i,j)}) := s_t^{(i,j)}+\Delta$, producing counterfactual targets such as $o_{t+1}^{cf}$ (or latent counterparts)\cite{pearl2009causality}.

\noindent\textbf{Unknowns and learning goal.}
We learn three coupled functions (with parameters $(\theta,\psi,\phi)$): encoder $E_\theta$ (observation $\rightarrow$ object-centric latent), causal branch $C_\psi$ (latent $\rightarrow$ causally regularized latent), and action-conditioned transition $F_\phi$ (latent/action $\rightarrow$ next latent). The learning goal is to obtain representations and dynamics that are accurate on factual transitions and remain reliable under interventions, i.e., strong counterfactual performance under $\mathrm{do}$-queries with causal selectivity rather than diffuse latent changes.

\noindent\textbf{Evaluation target.}
We evaluate with retrieval metrics H@1, MRR, CF-H@1, and CF-MRR (formal definitions are given in Sec.~\ref{subsec:metrics}); factual H@1/MRR are reported at 1/5/10-step horizons.

\subsection{Method Overview}
\label{sec:method_overview}

Figure~\ref{fig:method_overview} presents the overall architecture, including the main prediction pipeline and the internal structure of the CausalVAE branch.
At a high level, the model maps visual observations to object-centric latent states, refines these states with a causal structural module, and predicts future dynamics conditioned on actions.

Given an input observation $o_t$, the encoder produces an object-centric latent representation
\begin{equation}
z_t = E_\theta(o_t), \quad z_t \in \mathbb{R}^{K \times d},
\label{eq:overview_encoder}
\end{equation}
where $K$ denotes the number of object slots and $d$ is the latent dimension per slot.
This latent state is then processed by a causal branch to obtain a structurally regularized latent representation
\begin{equation}
\tilde{z}_t = C_\psi(z_t),
\label{eq:overview_causal_refine}
\end{equation}
where $C_\psi(\cdot)$ corresponds to the CausalVAE-inspired transformation shown in Fig.~\ref{fig:method_overview}.

To couple predictive accuracy with causal structure, the transition module takes action-conditioned latent input and predicts the next latent state:
\begin{equation}
\hat{z}_{t+1} = F_\phi(\tilde{z}_t, a_t).
\label{eq:overview_transition}
\end{equation}
Here $\tilde{z}_t$ denotes the latent representation fed into the transition model.  
Its concrete instantiation is specified later in the training strategy section.

Figure~\ref{fig:method_overview} contains two complementary views. The first view shows the end-to-end inference pipeline
\begin{equation}
o_t \rightarrow z_t \rightarrow \tilde{z}_t \xrightarrow[]{a_t} \hat{z}_{t+1},
\label{eq:overview_pipeline}
\end{equation}
which is the path used for factual and counterfactual prediction.
The second view expands the CausalVAE branch, where latent variables are encoded into a structured causal space and decoded back to latent dynamics space under structural regularization constraints.

This design combines two strengths from prior lines of work.
From structured world models~\cite{kipf2020cswm,ke2021systematic}, it inherits action-conditioned latent transition modeling and retrieval-oriented dynamics evaluation.
From CausalVAE-style modeling~\cite{yang2020causalvae}, it inherits structural causal regularization that encourages interpretable and intervention-sensitive representations.
As a result, the framework jointly targets factual prediction quality and counterfactual robustness within a unified latent dynamics architecture.

\subsection{CausalVAE Branch}
\label{sec:causalvae_branch}

Given the encoder latent $z_t$ (Eq.~\eqref{eq:overview_encoder}), our CausalVAE branch produces a structurally regularized latent $\tilde{z}_t$ (Eq.~\eqref{eq:overview_causal_refine}), which is then used by the transition model in Eq.~\eqref{eq:overview_transition}.
The overall forward relation is consistent with Fig.~\ref{fig:method_overview}:
\begin{equation}
o_t \rightarrow z_t \xrightarrow{\text{CausalVAE}} \tilde{z}_t \xrightarrow[]{a_t} \hat{z}_{t+1}.
\label{eq:cvae_forward_chain}
\end{equation}

The branch follows the CausalVAE principle~\cite{yang2020causalvae}: latent factors are organized through a structural causal transformation parameterized by a DAG matrix $A$, and then mapped back to latent dynamics space.
In our implementation, DAG/masking-style structural constraints are applied inside the branch to encourage causally meaningful factorization before decoding to $\tilde{z}_t$. Importantly, this causal branch is a plug-in module: it operates on latent representations and can be attached to different latent world-model backbones (e.g., C-SWM-style, GNN-style, or other encoder--transition pipelines) without changing their base transition architecture.

The training objective of this branch is
\begin{equation}
\mathcal{L}_{\mathrm{stage2}}
=
\lambda_1\mathcal{L}_{\mathrm{rec}}
+\lambda_2\mathcal{L}_{\mathrm{KL}}
+\lambda_3\mathcal{L}_{\mathrm{align}}
+\lambda_4\mathcal{L}_{\mathrm{DAG}}
+\lambda_5\mathcal{L}_{\mathrm{Mask}},
\label{eq:cvae_stage2_obj}
\end{equation}
where $\mathcal{L}_{\mathrm{rec}}$ reconstructs latent representation, $\mathcal{L}_{\mathrm{KL}}$ regularizes the variational posterior, $\mathcal{L}_{\mathrm{align}}$ aligns to available state supervision (when $s_t$ is available), and $\mathcal{L}_{\mathrm{DAG}}$ enforces acyclicity of $A$.

When simulator states are available, we impose an auxiliary alignment objective to map latent dynamics to physically meaningful state variables.
Let $g_{\mathrm{align}}(\cdot)$ denote the alignment head and $s_t$ the ground-truth state.
Given the CausalVAE output latent $\tilde{z}_t$, we define
\begin{equation}
\mathcal{L}_{\mathrm{align}}
=
\left\|
g_{\mathrm{align}}(\tilde{z}_t)-s_t
\right\|_2^2.
\label{eq:cvae_align}
\end{equation}
This term is only used during training and is weighted by $\lambda_{3}$ in Eq.~\eqref{eq:cvae_stage2_obj}.
At inference time, the model does not require $s_t$.
When $s_t$ (or concept labels) is unavailable, we set $\lambda_{3}=0$ and drop the state-conditioned terms in $\mathcal{L}_{\mathrm{KL}}$/$\mathcal{L}_{\mathrm{Mask}}$, optimizing only the unsupervised structural terms.

For the DAG constraint we use the standard smooth acyclicity penalty:
\begin{equation}
\mathcal{L}_{\mathrm{DAG}} = \mathrm{tr}\!\left(\exp(A \odot A)\right)-d,
\label{eq:cvae_dag}
\end{equation}
with $d$ the structural latent dimension~\cite{zhengDAGsNOTEARS2018}.

Following the CausalVAE formulation~\cite{yang2020causalvae}, we instantiate the remaining terms as:
\begin{align}
\mathcal{L}_{\mathrm{rec}}
&=
\left\|
\tilde{z}_t-z_t
\right\|_2^2, \quad \tilde{z}_t:=C_\psi(z_t),
\label{eq:cvae_lrec}
\\
\mathcal{L}_{\mathrm{KL}}
&=
\mathbb{E}\!\left[
\alpha_{\mathrm{KL}}\,\mathrm{KL}\!\left(q_\psi(\cdot\mid z_t)\,\|\,\mathcal N(0,I)\right)
+\sum_{i=1}^{d}\mathrm{KL}\!\left(q_\psi(\cdot\mid z_t,A)_i\,\|\,p(\cdot\mid s_t)_i\right)
\right],
\label{eq:cvae_lkl}
\\
\mathcal{L}_{\mathrm{Mask}}
&=
\mathbb{E}\!\left[
\sum_{i=1}^{d}\mathrm{KL}\!\left(q_\psi(\cdot\mid z_t,A,\mathrm{mask})_i\,\|\,p(\cdot\mid s_t)_i\right)
+\left\|g_{\mathrm{mask}}(z_t)-y_t\right\|_2^2
\right],
\label{eq:cvae_lmask}
\end{align}
where $q_\psi(\cdot\mid z_t)$ is the approximate posterior induced by the CausalVAE branch from encoder latent $z_t$, $q_\psi(\cdot\mid z_t,A)_i$ and $q_\psi(\cdot\mid z_t,A,\mathrm{mask})_i$ denote its $i$-th concept component after DAG and mask transformations, and $p(\cdot\mid s_t)_i$ is the corresponding state-conditioned prior component (when $s_t$ is available). Here $A$ is the learned DAG matrix, $d$ is the structural latent dimension (consistent with Eq.~\eqref{eq:cvae_dag}), $\mathrm{KL}(\cdot\|\cdot)$ is the Kullback--Leibler divergence, $g_{\mathrm{mask}}(\cdot)$ is the mask-branch predictor, and $y_t$ is concept-level supervision derived from available state labels. Following CausalVAE~\cite{yang2020causalvae}, we set $\alpha_{\mathrm{KL}}=0.3$ in our implementation.

\paragraph{Identifiability of the Causal Branch.}
\label{sec:identifiability}
We analyze identifiability for the Stage-2 structural learner, where the world-model backbone (encoder/transition) is frozen and the CausalVAE branch is optimized with
\begin{equation}
\label{eq:stage2_obj}
\min_{\psi,A,g_{\mathrm{align}}}\;
\mathcal{L}_{\mathrm{stage2}}
=\lambda_1\mathcal{L}_{\mathrm{rec}}
+\lambda_2\mathcal{L}_{\mathrm{KL}}
+\lambda_3\mathcal{L}_{\mathrm{align}}
+\lambda_4\mathcal{L}_{\mathrm{DAG}}
+\lambda_5\mathcal{L}_{\mathrm{Mask}}
\end{equation}
with $\tilde{z}_t=C_\psi(z_t)$ and $\mathcal{L}_{\mathrm{DAG}}$ enforcing acyclicity. \cite{yang2020causalvae}

\noindent\textbf{Assumptions.}
(I1) \emph{Stage-2 decoupling:} during Stage-2, transition/backbone parameters are frozen, so \eqref{eq:stage2_obj} is the only objective governing $(\psi,A,g_{\mathrm{align}})$.
(I2) \emph{Latent invertibility:} there exists an invertible map $g$ such that $z_t=g(v_t)$ for underlying causal state $v_t\in\mathbb R^d$.
(I3) \emph{Realizability + DAG:} there exists a ground-truth DAG adjacency $A^\star$ (acyclic) and parameters $(\psi^\star,g_{\mathrm{align}}^\star)$ attaining the population-risk minimum of \eqref{eq:stage2_obj}, and $\mathcal{L}_{\mathrm{DAG}}(A)=0$ iff $A$ is acyclic.
(I4) \emph{Alignment anchoring + scale normalization:} at the population optimum,
\begin{equation}
\label{eq:align_anchor}
g_{\mathrm{align}}(\tilde{z}_t)=s_t \quad \text{a.s.},
\end{equation}
and the coordinate system of $\tilde{z}_t$ is fixed by scale normalization (e.g., per-dimension zero-mean/unit-variance). Moreover, anchoring is axis-fixing: if an invertible $T:\mathbb R^d\!\to\!\mathbb R^d$ satisfies $g_{\mathrm{align}}(T(\tilde{z}_t))=g_{\mathrm{align}}(\tilde{z}_t)$ a.s. under the same normalization, then $T$ must be the identity.
(I5) \emph{Population/global optimum:} infinite-sample and exact optimization.
Why these assumptions are reasonable in our setting (and when they may fail) is discussed in App.~\ref{app:ident_assumptions}.

\begin{theorem}[Alignment-anchored identifiability]
\label{thm:align_ident}
Under (I1)--(I5), the Stage-2 optimizer identifies an adjacency $\hat A$ that matches $A^\star$ in the alignment-anchored coordinate system, i.e., $\hat A=A^\star$ in that coordinate frame.
\end{theorem}
This statement is conditional on (I1)--(I5) and should be interpreted as a population-level characterization, not a finite-sample optimization guarantee.

\noindent\textit{Proof.} See App.~\ref{app:ident_proof}.

\begin{figure}[t]
    \centering
    \includegraphics[width=0.70\linewidth]{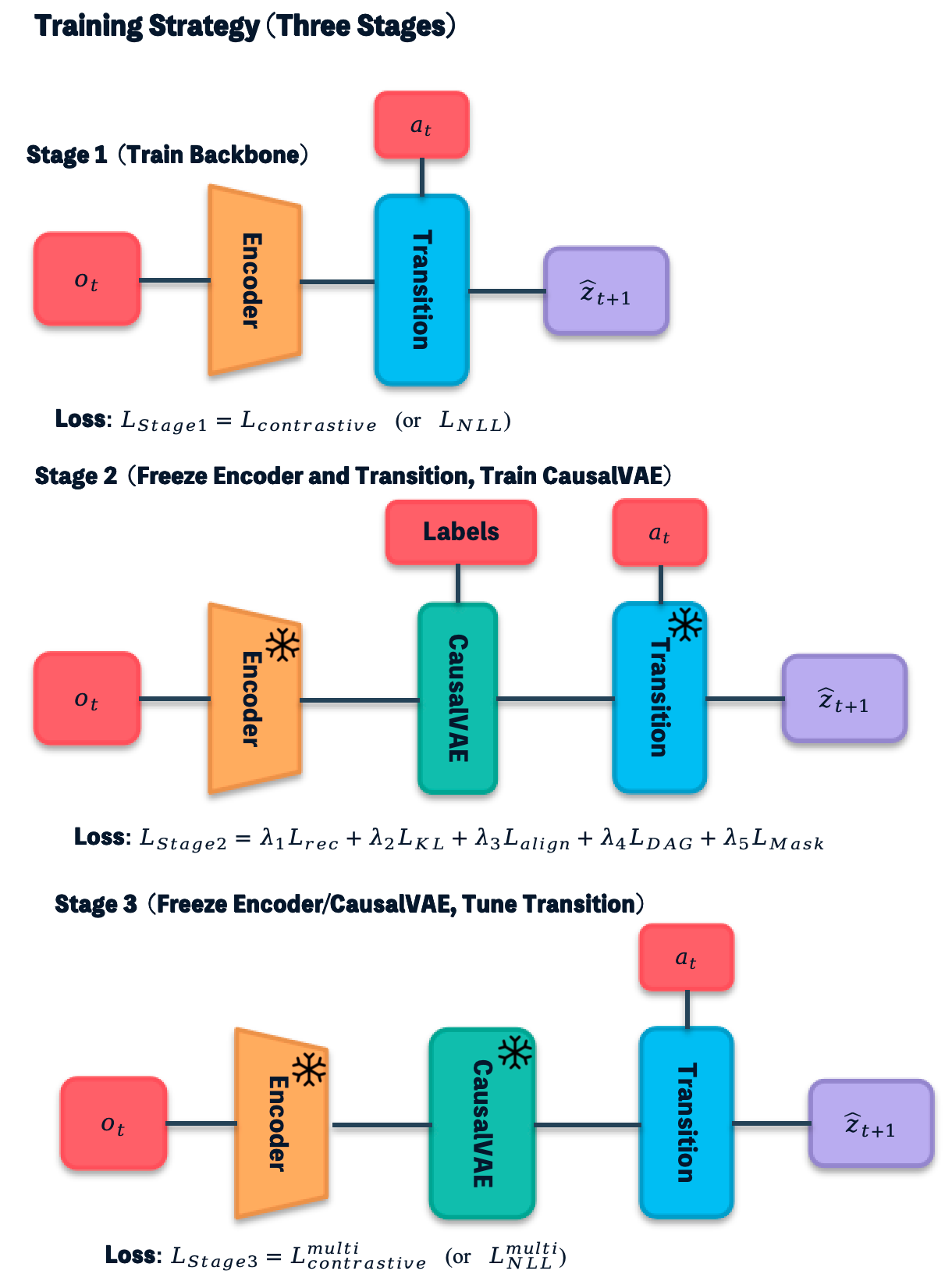}
    \caption{Three-stage training strategy.}
    \label{fig:training_strategy}
\end{figure}

\noindent\textbf{Relation to CausalVAE identifiability theorem in the original paper~\cite{yang2020causalvae}.}
The identifiability proof in CausalVAE relies on an auxiliary variable $u$ through a conditional prior $p(z\mid u)$, leading to $\sim$-identifiability of the generative parameters under that training setting.
In our model, we do \emph{not} assume access to $u$ nor optimize a conditional prior; instead, we use only the alignment loss
$\mathcal{L}_{\mathrm{align}}=\|g_{\mathrm{align}}(\tilde{z}_t)-s_t\|_2^2$
as weak supervision to anchor the latent coordinates.
Therefore, our identifiability argument does not directly invoke CausalVAE's Theorem~1; it follows a different route in which alignment eliminates the permutation/scale (and more generally reparameterization) ambiguities, after which the DAG-constrained structural learner identifies the adjacency in the anchored coordinate system.
We cite CausalVAE for (i) the formal ambiguity class captured by $\sim$-identifiability and (ii) its discussion on identifiability of the causal graph in the causal layer~\cite{yang2020causalvae}.

\subsection{Transition Modeling}
\label{sec:transition_modeling}

The transition module models action-conditioned dynamics in latent space.
Following Sec.~\ref{sec:method_overview}, the encoder latent and causal-refined latent are defined in
Eq.~\eqref{eq:overview_encoder} and Eq.~\eqref{eq:overview_causal_refine}, respectively.
Next-state prediction then follows Eq.~\eqref{eq:overview_transition}.

The supervision target is the latent encoding of the next-step observation:
\begin{equation}
z_{t+1} = E_{\theta}(o_{t+1}).
\label{eq:trans_target}
\end{equation}

To stabilize long-horizon prediction, we use residual dynamics parameterization:
\begin{equation}
\Delta z_t = f_{\phi}(\tilde{z}_t, a_t), \qquad
\hat{z}_{t+1} = \tilde{z}_t + \Delta z_t.
\label{eq:trans_residual}
\end{equation}

For multi-step rollout, predictions are generated recursively:
\begin{equation}
\hat{z}_{t+k+1} = F_{\phi}(\hat{z}_{t+k}, a_{t+k}), \quad k \ge 0,
\label{eq:trans_rollout}
\end{equation}
with $\hat{z}_t \leftarrow \tilde{z}_t$, and targets
\begin{equation}
z_{t+k} = E_{\theta}(o_{t+k}).
\label{eq:trans_rollout_target}
\end{equation}

The rollout objective is
\begin{equation}
\mathcal{L}_{\mathrm{multistep}}
=
\frac{1}{H}\sum_{k=1}^{H}\ell\!\left(\hat{z}_{t+k}, z_{t+k}\right),
\label{eq:multistep_loss}
\end{equation}
where $H$ is the rollout horizon, and $\ell(\cdot,\cdot)$ is instantiated as either a contrastive objective (as in C-SWM-style latent energy matching~\cite{kipf2020cswm}) or a latent negative log-likelihood objective (as in likelihood-based world models~\cite{ha2018worldmodels,hafner2019dreamer}).
For the contrastive case, let $z^-_{t+k}$ denote a negative latent target at step $t{+}k$ (sampled from non-matching futures in the candidate pool/batch), and let $m>0$ be the margin. We use
\begin{equation}
\ell_{\mathrm{con}}\!\left(\hat z,z\right)
=
\left\|\hat z-z\right\|_2^2
 +
\max\!\left(0,\; m-\left\|\hat z-z^- \right\|_2^2\right).
\label{eq:contrastive_latent_loss}
\end{equation}
For the NLL case, assuming an isotropic Gaussian decoder in latent space,
\begin{equation}
\ell_{\mathrm{NLL}}\!\left(\hat z,z\right)
=
-\log p_\phi\!\left(z\mid \hat z\right),\quad
p_\phi\!\left(z\mid \hat z\right)=\mathcal N\!\left(z;\hat z,\sigma^2 I\right),
\label{eq:nll_latent_loss}
\end{equation}
which is equivalent (up to constants) to a scaled squared-error term.

\subsection{Training Strategy}
\label{sec:training_strategy}

We adopt a three-stage training pipeline (Fig.~\ref{fig:training_strategy}) to decouple
(i) backbone dynamics pretraining, (ii) causal structure learning, and
(iii) long-horizon transition refinement.

\noindent\textbf{Stage 1 (Backbone pretraining).}
We optimize the encoder-transition backbone using one-step supervision defined in
Sec.~\ref{sec:method_overview}, i.e., $z_t=E_\theta(o_t)$ and
$\hat{z}_{t+1}=F_\phi(z_t,a_t)$.

\noindent\textbf{Stage 2 (Causal branch training).}
With backbone modules frozen, we train the CausalVAE branch in
Sec.~\ref{sec:causalvae_branch} using Eq.~\eqref{eq:cvae_stage2_obj},
including reconstruction, KL, alignment (when supervision is available), DAG, and mask terms. Details on how supervision is obtained and how objective slots are selected for each benchmark are provided in App.~\ref{app:supervision_slots}.

\noindent\textbf{Stage 3 (Transition Refinement with Alpha-Gated Fusion).}
We freeze the encoder and CausalVAE, and fine-tune the transition model with
\begin{equation}
z_t^{\mathrm{gate}}=(1-\alpha_t)z_t+\alpha_t \tilde{z}_t, \qquad \alpha_t\in[0,1].
\label{eq:alpha_gated_fusion}
\end{equation}
In implementation, we use an exponentially decayed mixing coefficient
($\alpha_t=\alpha_0\exp(-k_\alpha t)$), where $\alpha_0$ is the initial fusion weight and $k_\alpha$ controls decay speed. We optionally apply an additional data-dependent gate with $\alpha_t^{\mathrm{eff}}=g_t\alpha_t$ and $g_t=\sigma\!\left((\tau-\delta_t)\,\gamma\right)$, where $\delta_t=\|\tilde{z}_t-z_t\|_2$ and $(\tau,\gamma)$ are gate hyperparameters.
The transition update is $\hat{z}_{t+1}=F_\phi(z_t^{\mathrm{gate}},a_t)$, and multi-step training follows Sec.~\ref{sec:transition_modeling}
(Eq.~\eqref{eq:trans_rollout} and Eq.~\eqref{eq:multistep_loss}). We report an ablation over gate on/off in Sec.~\ref{subsec:ablation} (Tab.~\ref{tab:ablation}).
The rationale for using a three-stage schedule is also empirically validated in Sec.~\ref{subsec:ablation} via explicit w/3-stage vs.\ w/o-3-stage comparisons.

\section{Experiments}
\label{sec:experiments}

We evaluate whether injecting an explicit causal factorization via CausalVAE improves
\emph{counterfactual reliability} while preserving factual predictive performance.

\subsection{Benchmarks}
\label{subsec:benchmarks}

\begin{figure}[t]
    \centering
    \includegraphics[width=0.85\linewidth]{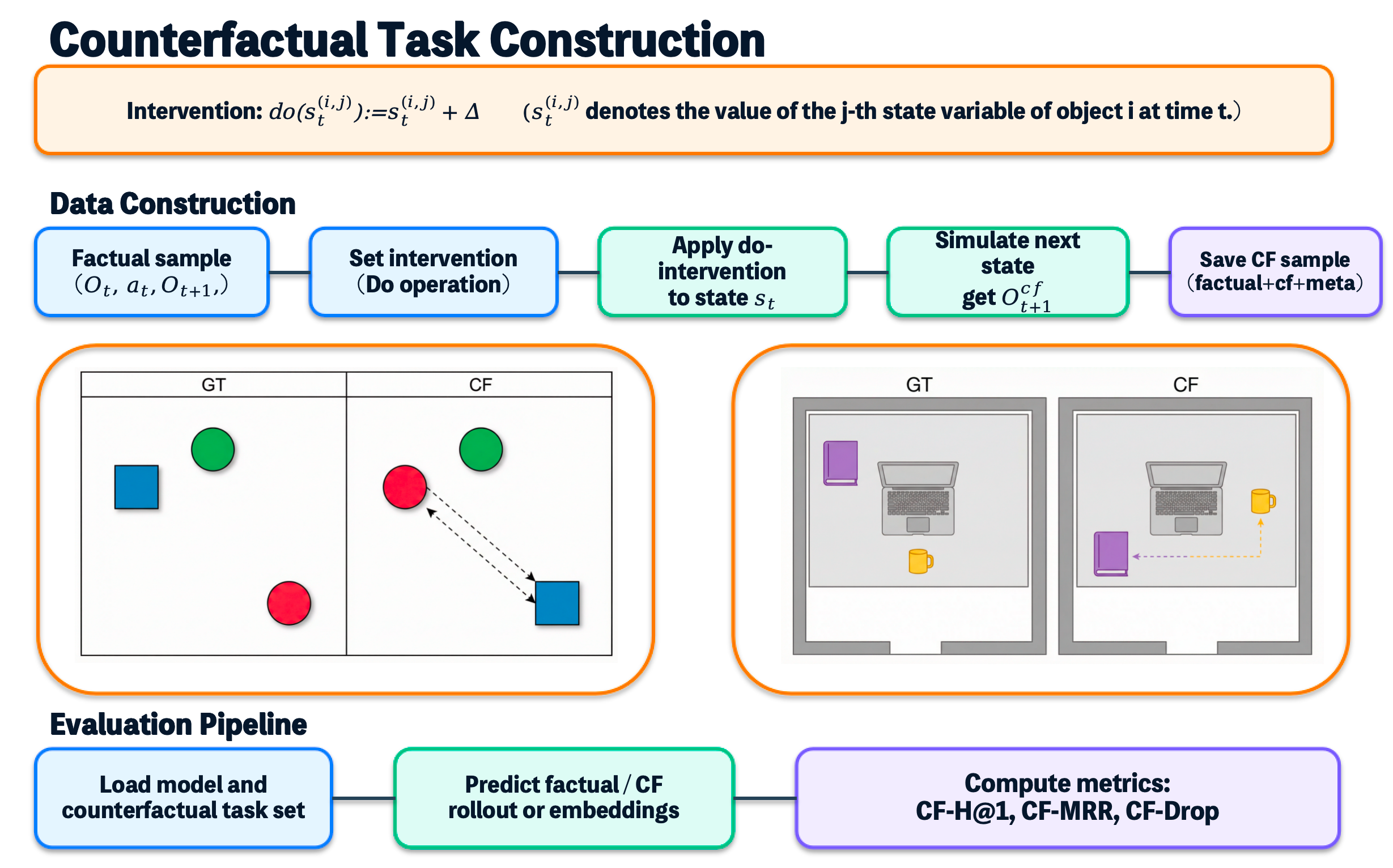}
    \caption{Counterfactual task construction and evaluation pipeline. From factual tuples $(o_t,a_t,o_{t+1})$, we apply interventions $do(s_t^{(i,j)}){:=}s_t^{(i,j)}+\Delta$, re-simulate to obtain $o_{t+1}^{cf}$, and evaluate retrieval on paired factual/counterfactual futures using H@1, MRR, CF-H@1, and CF-MRR (factual H@1/MRR at 1/5/10-step horizons).}
    \label{fig:cf_task_construction}
\end{figure}

We evaluate on a suite of environments spanning object-centric physics and perceptual manipulation in the causal discovery evaluation framework\cite{ke2021systematic}.
All benchmarks are evaluated under both factual and interventional (counterfactual) protocols, when ground-truth
counterfactual rollouts are available.
As illustrated in Fig.~\ref{fig:cf_task_construction}, counterfactual tasks are constructed by starting from factual tuples $(o_t,a_t,o_{t+1})$, applying a $do$-intervention to selected state variables at time $t$, re-simulating the next observation under the intervened state to obtain $o_{t+1}^{cf}$, and then evaluating retrieval on paired factual/counterfactual futures.
We evaluate four benchmark domains: Physics (3-body gravitation), 2D Shapes, 3D Cubes, and Chemistry. Physics supports clean state-level interventions (e.g., position/velocity), Shapes/Cubes provide object-level manipulation interventions (e.g., pose/push), and Chemistry provides mechanism-level interventions when available.
Detailed counterfactual benchmark construction (query-group counts, intervention variables, magnitude settings, and candidate-set protocol) is provided in App.~\ref{app:protocol_details}.

\subsection{Baselines}
\label{subsec:baselines}

\paragraph{Eight baselines (paired with +CausalVAE).}
We compare eight standard world-model baselines against matched \textbf{Baseline + CausalVAE} counterparts:

\smallskip
\noindent\textbf{AE-based:} AE\_NLL, AE\_Contrastive.\;\\
\textbf{VAE-based:} VAE\_NLL, VAE\_Contrastive.\;\\
\textbf{Structured latent dynamics:} Modular\_NLL, Modular\_Contrastive.\;\\
\textbf{Graph dynamics:} GNN\_NLL, GNN\_Contrastive.
Detailed baseline network architectures and objective definitions (NLL/contrastive) are provided in App.~\ref{app:protocol_details}.

\paragraph{Our model.}
Starting from each baseline backbone, we insert a latent CausalVAE structural layer and train with three stages (S1/S2/S3) for factual dynamics, structural constraint learning, and counterfactual refinement.
This S1/S2/S3 schedule applies to \textbf{Baseline+CausalVAE} models; pure baselines follow the baseline objective settings (NLL, contrastive) described in App.~\ref{app:protocol_details}.

\paragraph{Training protocol.}
\label{subsec:training}
We use a unified training protocol for fair paired comparison (Baseline vs.\ Baseline+CausalVAE): same data splits, seeds, optimizer family, and rollout settings.
The full hyperparameter and configuration details, together with benchmark-specific settings, are provided in App.~\ref{app:protocol_details}, including detailed AE/VAE/GNN/Modular architecture and loss-function specifications aligned with the reference baseline protocol.

\paragraph{Intervention and retrieval protocols.}
\label{subsec:protocol}
We evaluate counterfactual reasoning via controlled interventions at time $t_0$ and retrieval over factual/counterfactual future candidates under a shared protocol.
Detailed intervention construction (single/multi-target, magnitude sweeps, candidate sampling, and horizon settings) is deferred to App.~\ref{app:protocol_details}.

\subsection{Metrics}
\label{subsec:metrics}

We report four retrieval metrics: \textbf{H@1}, \textbf{MRR}, \textbf{CF-H@1}, and \textbf{CF-MRR}. For query $i$ with factual rank $r_i$ and counterfactual rank $r_i^{cf}$, we use
$\mathrm{H@1}=\frac{1}{M}\sum_{i=1}^{M}\mathbf{1}[r_i=1]$,
$\mathrm{MRR}=\frac{1}{M}\sum_{i=1}^{M}\frac{1}{r_i}$,
$\mathrm{CF\text{-}H@1}=\frac{1}{M}\sum_{i=1}^{M}\mathbf{1}[r_i^{cf}=1]$, and
$\mathrm{CF\text{-}MRR}=\frac{1}{M}\sum_{i=1}^{M}\frac{1}{r_i^{cf}}$.
Factual H@1/MRR are reported at 1/5/10-step horizons using the same definitions on each horizon-specific retrieval task. CF-H@1/CF-MRR are reported for the counterfactual step.

\subsection{Benchmarking Results}
\label{subsec:main_results}

\paragraph{Factual retrieval comparison.}
Tab.~\ref{tab:summary_main} reports the factual retrieval performance
(H@1 and MRR) of all baseline families and their +CausalVAE variants
across all benchmarks.

\begin{table}[!htbp]
    \centering
\scriptsize
\setlength{\tabcolsep}{2.4pt}
\renewcommand{\arraystretch}{0.92}
    \caption{Comparison on factual retrieval metrics.
    Each cell reports Baseline/+CausalVAE. Bold indicates the best value in each column.}
    \label{tab:summary_main}
\resizebox{0.99\linewidth}{!}{%
    \begin{tabular}{llcccccc}
    \toprule
    & & \multicolumn{2}{c}{1 Step} & \multicolumn{2}{c}{5 Steps} & \multicolumn{2}{c}{10 Steps} \\
    \cmidrule(lr){3-4}\cmidrule(lr){5-6}\cmidrule(lr){7-8}
    Benchmark & Method
    & H@1 $\uparrow$ & MRR $\uparrow$
    & H@1 $\uparrow$ & MRR $\uparrow$
    & H@1 $\uparrow$ & MRR $\uparrow$ \\
    \midrule
    
\multirow{8}{*}{Physics 3-body}
& AE\_Contrastive
& \num{89.00}/\num{98.00} & \num{94.25}/\num{99.00}
& \num{10.00}/\num{9.00}  & \num{27.91}/\num{25.72}
& \num{1.00}/\num{3.00}   & \num{6.59}/\num{9.31} \\

& AE\_NLL
& \num{92.00}/\num{94.00} & \num{95.67}/\num{96.83}
& \num{4.00}/\num{4.00}   & \num{16.58}/\num{16.39}
& \num{0.00}/\num{1.00}   & \num{8.93}/\num{9.04} \\

& VAE\_Contrastive
& \num{70.00}/\num{91.00} & \num{83.09}/\num{95.33}
& \num{1.00}/\num{3.00}   & \num{4.96}/\num{10.44}
& \num{2.00}/\num{2.00}   & \num{5.74}/\num{5.73} \\

& VAE\_NLL
& \num{92.00}/\num{94.00} & \num{95.83}/\num{97.00}
& \num{1.00}/\num{2.00}   & \num{10.19}/\num{9.46}
& \num{0.00}/\num{0.00}   & \num{3.61}/\num{3.64} \\

& Modular\_Contrastive
& \num{90.00}/\num{86.00} & \num{94.42}/\num{92.33}
& \num{20.00}/\num{8.00}  & \num{36.31}/\num{21.23}
& \num{3.00}/\num{2.00}   & \num{10.44}/\num{6.37} \\

& Modular\_NLL
& \num{96.00}/\num{95.00} & \num{97.83}/\num{97.50}
& \num{11.00}/\num{11.00} & \num{30.28}/\num{28.36}
& \num{1.00}/\num{3.00}   & \num{9.38}/\num{9.99} \\

& GNN\_Contrastive
& \num{88.00}/\textbf{\num{99.00}} & \num{93.37}/\textbf{\num{99.50}}
& \num{25.00}/\textbf{\num{36.00}} & \num{37.77}/\textbf{\num{54.58}}
& \num{3.00}/\textbf{\num{7.00}}   & \num{11.89}/\textbf{\num{17.00}} \\

& GNN\_NLL
& \num{98.00}/\num{99.00} & \num{99.00}/\num{99.50}
& \num{9.00}/\num{17.00}  & \num{29.36}/\num{33.32}
& \textbf{\num{7.00}}/\num{5.00} & \num{15.27}/\num{12.57} \\

\midrule
\multirow{8}{*}{Chemistry}
& AE\_Contrastive
& \num{99.91}/\num{99.90} & \num{99.95}/\num{99.95}
& \num{99.84}/\num{99.81} & \num{99.92}/\num{99.91}
& \textbf{\num{99.92}}/\num{99.33} & \textbf{\num{99.96}}/\num{99.66} \\

& AE\_NLL
& \num{99.86}/\num{99.89} & \num{99.93}/\num{99.94}
& \num{98.17}/\num{96.77} & \num{98.98}/\num{98.07}
& \num{93.24}/\num{85.73} & \num{95.70}/\num{90.13} \\

& VAE\_Contrastive
& \num{99.95}/\num{99.96} & \num{99.98}/\num{99.98}
& \num{99.94}/\num{99.73} & \num{99.97}/\num{99.87}
& \textbf{\num{99.92}}/\num{98.60} & \textbf{\num{99.96}}/\num{99.30} \\

& VAE\_NLL
& \num{96.03}/\num{97.47} & \num{97.65}/\num{98.65}
& \num{88.50}/\num{74.59} & \num{92.69}/\num{82.31}
& \num{78.20}/\num{25.32} & \num{85.05}/\num{36.50} \\

& Modular\_Contrastive
& \num{99.50}/\num{16.03} & \num{99.75}/\num{33.95}
& \num{77.97}/\num{6.48}  & \num{88.06}/\num{17.97}
& \num{30.50}/\num{3.76}  & \num{53.24}/\num{11.83} \\

& Modular\_NLL
& \num{99.97}/\textbf{\num{99.99}} & \num{99.98}/\textbf{\num{99.99}}
& \num{99.43}/\num{98.63} & \num{99.69}/\num{99.24}
& \num{95.74}/\num{67.20} & \num{97.18}/\num{76.73} \\

& GNN\_Contrastive
& \num{99.95}/\num{99.95} & \num{99.98}/\num{99.98}
& \textbf{\num{99.96}}/\num{99.74} & \textbf{\num{99.98}}/\num{99.87}
& \num{99.86}/\num{99.48} & \num{99.93}/\num{99.74} \\

& GNN\_NLL
& \num{99.95}/\num{99.91} & \num{99.98}/\num{99.95}
& \num{99.94}/\num{99.38} & \num{99.97}/\num{99.66}
& \num{99.91}/\num{89.28} & \num{99.95}/\num{92.74} \\
\midrule
\multirow{8}{*}{2D Shapes}
& AE\_Contrastive
& \num{94.01}/\num{94.03} & \num{96.58}/\num{96.75}
& \num{58.84}/\num{58.16} & \num{71.95}/\num{71.43}
& \num{28.34}/\num{27.44} & \num{43.42}/\num{42.79} \\

& AE\_NLL
& \num{99.42}/\num{98.83} & \num{99.60}/\num{99.25}
& \num{81.31}/\num{77.28} & \num{86.45}/\num{83.26}
& \num{42.57}/\num{40.37} & \num{52.23}/\num{50.06} \\

& VAE\_Contrastive
& \num{76.51}/\num{87.70} & \num{84.81}/\num{93.35}
& \num{20.30}/\num{1.75}  & \num{34.95}/\num{7.11}
& \num{5.58}/\num{0.29}   & \num{13.67}/\num{1.37} \\

& VAE\_NLL
& \num{59.45}/\num{60.37} & \num{66.39}/\num{67.29}
& \num{9.84}/\num{10.82}  & \num{14.11}/\num{14.93}
& \num{2.00}/\num{2.14}   & \num{3.55}/\num{3.72} \\

& Modular\_Contrastive
& \num{3.25}/\num{2.16}   & \num{14.23}/\num{9.06}
& \num{0.73}/\num{0.19}   & \num{4.42}/\num{1.63}
& \num{0.33}/\num{0.08}   & \num{2.48}/\num{0.80} \\

& Modular\_NLL
& \num{99.80}/\num{99.62} & \num{99.89}/\num{99.78}
& \num{85.69}/\num{83.88} & \num{90.43}/\num{88.25}
& \num{43.04}/\num{47.10} & \num{53.76}/\num{56.21} \\

& GNN\_Contrastive
& \textbf{\num{99.91}}/\num{99.90} & \textbf{\num{99.95}}/\num{99.95}
& \textbf{\num{98.89}}/\num{97.60} & \textbf{\num{99.40}}/\num{98.60}
& \textbf{\num{96.11}}/\num{86.02} & \textbf{\num{97.52}}/\num{90.72} \\

& GNN\_NLL
& \num{94.42}/\num{95.73} & \num{96.90}/\num{97.73}
& \num{43.77}/\num{48.58} & \num{54.75}/\num{60.43}
& \num{17.14}/\num{20.79} & \num{25.67}/\num{30.64} \\
\midrule
\multirow{8}{*}{3D Cubes}
& AE\_Contrastive
& \num{68.37}/\num{69.69} & \num{76.65}/\num{78.64}
& \num{19.49}/\num{18.79} & \num{31.29}/\num{31.00}
& \num{7.35}/\num{5.72} & \textbf{\num{15.50}}/\num{13.61} \\

& AE\_NLL
& \num{78.90}/\textbf{\num{79.64}} & \num{85.18}/\textbf{\num{85.79}}
& \num{26.45}/\textbf{\num{27.42}} & \num{36.38}/\textbf{\num{37.44}}
& \num{7.54}/\textbf{\num{8.23}} & \num{13.04}/\num{13.82} \\

& VAE\_Contrastive
& \num{65.25}/\num{66.09} & \num{74.08}/\num{74.73}
& \num{14.72}/\num{17.49} & \num{25.15}/\num{28.54}
& \num{4.36}/\num{6.71} & \num{9.88}/\num{13.87} \\

& VAE\_NLL
& \num{55.44}/\num{52.58} & \num{61.86}/\num{58.36}
& \num{10.69}/\num{9.36} & \num{15.21}/\num{12.97}
& \num{2.27}/\num{1.45} & \num{4.19}/\num{2.85} \\

& Modular\_Contrastive
& \num{3.18}/\num{11.32} & \num{9.06}/\num{23.37}
& \num{0.39}/\num{0.42} & \num{1.78}/\num{2.12}
& \num{0.21}/\num{0.13} & \num{1.09}/\num{0.78} \\

& Modular\_NLL
& \num{64.77}/\num{65.09} & \num{73.26}/\num{73.52}
& \num{14.94}/\num{15.88} & \num{23.17}/\num{24.33}
& \num{2.95}/\num{3.49} & \num{6.52}/\num{7.76} \\

& GNN\_Contrastive
& \num{60.88}/\num{59.38} & \num{69.17}/\num{68.34}
& \num{15.71}/\num{2.67} & \num{24.89}/\num{7.67}
& \num{5.45}/\num{0.13} & \num{11.09}/\num{0.86} \\

& GNN\_NLL
& \num{59.61}/\num{60.13} & \num{67.00}/\num{67.44}
& \num{12.23}/\num{13.34} & \num{19.04}/\num{20.39}
& \num{2.66}/\num{3.06} & \num{5.71}/\num{6.59} \\
\bottomrule
\end{tabular}}
\end{table}
\FloatBarrier

\paragraph{Counterfactual retrieval comparison.}
Tab.~\ref{tab:summary_cf} reports the counterfactual retrieval performance
(CF-H@1 and CF-MRR) of all baseline families and their +CausalVAE variants
across all benchmarks.

\paragraph{Results analysis.}
Across benchmarks (Tabs.~\ref{tab:summary_main} and \ref{tab:summary_cf}), CausalVAE usually preserves factual retrieval while improving counterfactual retrieval.
The strongest evidence is on Physics: CF-H@1 jumps from 11.0 to 41.0 for GNN\_NLL, and AE\_NLL gains +30.5.
On Chemistry and 2D/3D, multiple backbones (VAE, Modular, GNN\_Contrastive) still gain roughly +9 to +21 CF-H@1 points.

The pattern is therefore not a small average effect; it is a large improvement on intervention-focused metrics in several settings.
At the same time, gains are not universal: some factual and counterfactual cells drop (e.g. Chemistry Modular\_Contrastive in factual metrics), showing a clear backbone-domain interaction.
This sharpens the takeaway: the plug-in is most valuable when baseline dynamics are intervention-fragile, but it should be selected per backbone rather than assumed to dominate uniformly.

\vspace{0em}

\subsection{Causal Analysis (Causal Discovery Results)}
\label{subsec:causal_analysis}
\label{sec:learnedA_vs_physics}

We report causal discovery results by comparing learned structural dependencies against physics-derived first-order interaction templates in the Physics benchmark. A key takeaway is that the learned structural matrix is not only predictive, but also recovers physically meaningful interaction patterns: it captures who influences whom and with what relative strength, consistent with the local first-order dynamics implied by the underlying system equations.
Let the continuous-time system be $\dot{s}(t)=f(s(t))$. Using one-step Euler discretization with step size $\Delta t$, we write
\begin{equation}
s_{t+1}^{\mathrm{true}}=F_{\mathrm{true}}(s_t)\approx s_t+\Delta t\,f(s_t).
\end{equation}
To obtain a local first-order form, we linearize $F_{\mathrm{true}}$ around a reference state $s^\star$:
$F_{\mathrm{true}}(s_t)=F_{\mathrm{true}}(s^\star)+\left.\frac{\partial F_{\mathrm{true}}}{\partial s}\right|_{s^\star}(s_t-s^\star)+\mathcal{O}\!\left(\|s_t-s^\star\|^2\right)$.
Define the local Jacobian as $J(s^\star):=\left.\frac{\partial F_{\mathrm{true}}}{\partial s}\right|_{s^\star}$; from Euler form,
$J(s^\star)\approx I+\Delta t\left.\frac{\partial f}{\partial s}\right|_{s^\star}$ and
$J(s^\star)-I\approx\Delta t\left.\frac{\partial f}{\partial s}\right|_{s^\star}$.
For visualization of interaction strength (excluding identity carry-over), we define
$A_{\mathrm{GT}}:=\left|J(s^\star)-I\right|$ (element-wise absolute value), and compare
$A_{\mathrm{Learned}}$ with $A_{\mathrm{GT}}$ at the mechanism-trend level (global coupling/channel pattern), rather than enforcing strict element-wise equality.
This provides evidence that the model discovers an interpretable first-order physical law in latent space: dominant interaction channels and their relative ordering are aligned with the local linearized dynamics.

\begin{table}[!htbp]
    \centering
\tiny
\setlength{\tabcolsep}{2.4pt}
\renewcommand{\arraystretch}{0.92}
    \caption{Comparison on counterfactual retrieval metrics.
    Each cell reports Baseline/+CausalVAE. $\Delta$ indicates the absolute improvement over the baseline.}
    \label{tab:summary_cf}
    
\resizebox{0.92\linewidth}{!}{%
    \begin{tabular}{l l c c c c}
    \toprule
    Benchmark & Method & CF-H@1 $\uparrow$ & $\Delta$H@1 $\uparrow$ & CF-MRR $\uparrow$ & $\Delta$MRR $\uparrow$ \\
    \midrule
    
    \multirow{8}{*}{Physics 3-body}
& AE\_Contrastive & 10.00/24.00 & +14.00 & 49.33/51.77 & +2.44 \\
& AE\_NLL         & 10.50/41.00 & +30.50 & 52.25/66.92 & +14.67 \\

& VAE\_Contrastive & 8.50/8.50 & -- & 46.00/45.40 & -- \\
& VAE\_NLL         & 10.50/13.00 & +2.50 & 51.75/51.00 & -- \\

& Modular\_Contrastive & 22.50/25.00 & +2.50 & 59.83/50.62 & -- \\
& Modular\_NLL         & 14.50/22.00 & +7.50 & 55.83/57.50 & +1.67 \\

& GNN\_Contrastive & 11.50/15.00 & +3.50 & 52.96/53.17 & +0.21 \\
& GNN\_NLL         & 11.00/41.00 & +30.00 & 53.50/68.88 & +15.38 \\

\midrule

\multirow{8}{*}{Chemistry}
& AE\_Contrastive & 52.61/35.87 & -- & 59.47/55.81 & -- \\
& AE\_NLL         & 35.91/34.34 & -- & 55.82/54.12 & -- \\

& VAE\_Contrastive & 15.38/24.61 & +9.23 & 33.53/48.24 & +14.71 \\
& VAE\_NLL         & 15.50/25.26 & +9.76 & 33.62/45.80 & +12.18 \\

& Modular\_Contrastive & 13.05/27.49 & +14.44 & 33.22/48.67 & +15.45 \\
& Modular\_NLL         & 28.30/29.56 & +1.26 & 50.46/51.09 & +0.63 \\

& GNN\_Contrastive & 13.54/25.76 & +12.22 & 32.96/47.66 & +14.70 \\
& GNN\_NLL         & 25.96/23.70 & -- & 48.44/45.20 & -- \\

\midrule

\multirow{8}{*}{2D Shapes}
& AE\_Contrastive & 53.10/52.53 & -- & 70.97/70.86 & -- \\
& AE\_NLL         & 65.75/67.82 & +2.07 & 80.22/81.64 & +1.42 \\

& VAE\_Contrastive & 46.97/45.31 & -- & 66.15/65.74 & -- \\
& VAE\_NLL         & 19.70/19.40 & -- & 41.76/43.11 & +1.35 \\

& Modular\_Contrastive & 13.85/8.55 & -- & 32.39/22.92 & -- \\
& Modular\_NLL         & 62.80/68.35 & +5.55 & 78.46/81.88 & +3.42 \\

& GNN\_Contrastive & 67.06/68.50 & +1.44 & 81.11/81.94 & +0.83 \\
& GNN\_NLL         & 33.00/54.75 & +21.75 & 58.81/74.00 & +15.19 \\

\midrule

\multirow{8}{*}{3D Cubes}
& AE\_Contrastive & 29.90/33.20 & +3.30 & 53.02/55.40 & +2.38 \\
& AE\_NLL         & 39.90/41.65 & +1.75 & 60.41/63.11 & +2.70 \\

& VAE\_Contrastive & 27.15/27.65 & +0.50 & 50.08/51.01 & +0.93 \\
& VAE\_NLL         & 22.15/20.35 & -- & 38.32/41.01 & +2.69 \\

& Modular\_Contrastive & 7.90/16.75 & +8.85 & 19.07/33.80 & +14.73 \\
& Modular\_NLL         & 30.60/30.90 & +0.30 & 53.74/54.83 & +1.09 \\

& GNN\_Contrastive & 20.05/18.50 & -- & 44.28/42.57 & -- \\
& GNN\_NLL         & 28.45/25.90 & -- & 48.15/48.73 & +0.58 \\
\bottomrule
\end{tabular}}
\end{table}
\FloatBarrier

\noindent\textbf{Limitation.}
The above correspondence is inherently local and first-order. Higher-order nonlinear terms in the Taylor expansion are not captured by the linear structural matrix. Therefore, the discovered law is a first-order approximation rather than a full nonlinear governing equation.

\noindent\textbf{Visualization note.}
To make this comparison explicit, we visualize the learned and reference interaction patterns in Fig.~\ref{fig:gt_vs_learned}. We place this figure after the analytical discussion so that it serves as supporting evidence rather than interrupting the flow of the result narrative.

\begin{figure}[t]
    \centering
    \includegraphics[width=0.68\linewidth]{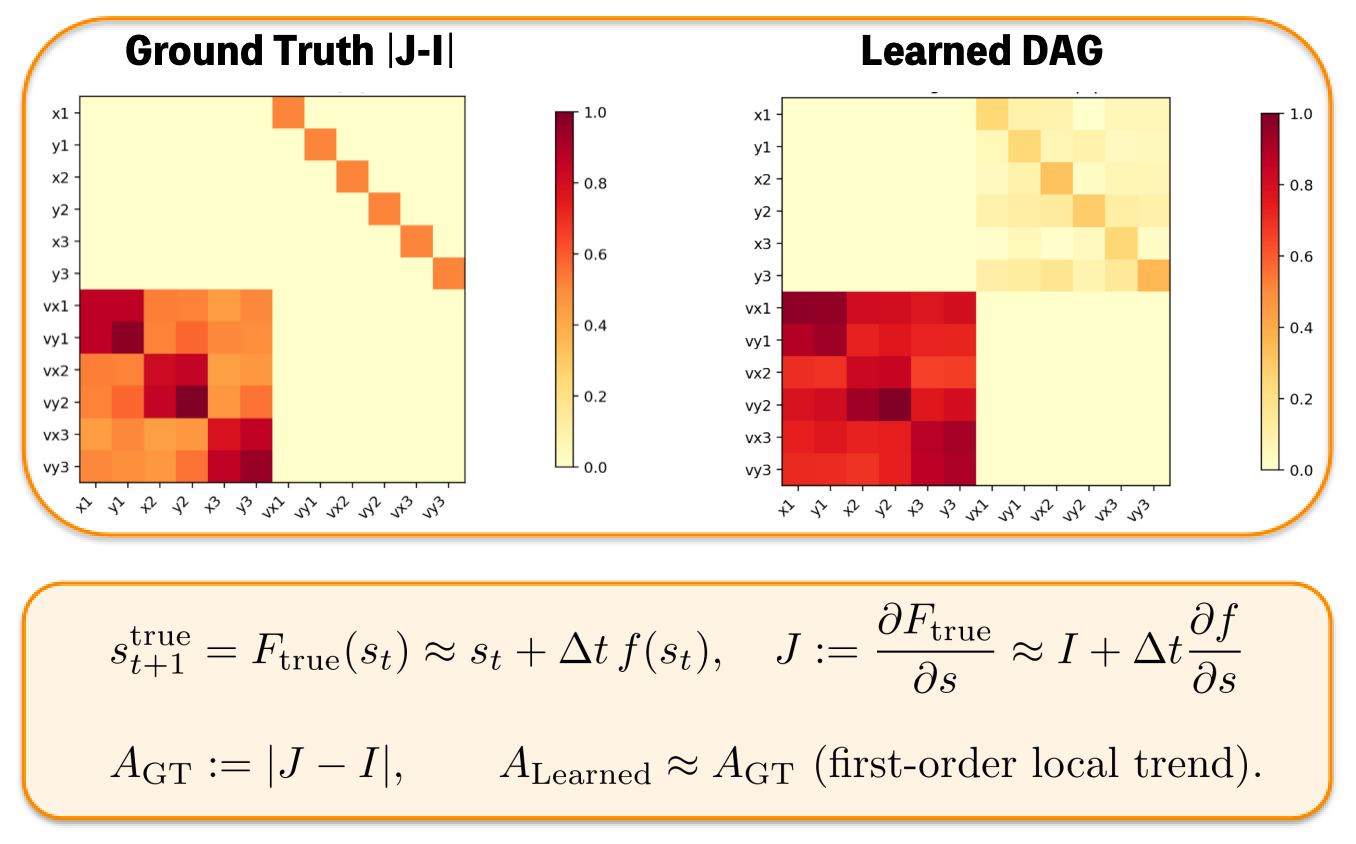}
    \caption{
    Comparison between learned structure and first-order physical template.
    The physical reference matrix is defined as $A_{\mathrm{GT}}=\left|J-I\right|$ from local linearization.
    The alignment indicates that the model recovers physically meaningful interaction trends in latent space, while remaining a first-order approximation rather than an exact nonlinear law.
    }
    \label{fig:gt_vs_learned}
\end{figure}
\FloatBarrier

\subsection{Ablation Results}
\label{subsec:ablation}



As summarized in Tab.~\ref{tab:ablation}, we ablate components that correspond directly to our staged design and rollout training choices on the Physics benchmark. Concretely, we compare \textbf{three-stage training} (S1+S2+S3) against \textbf{joint training} (w/o stage split), and include the corresponding \textbf{backbone baseline} (GNN\_Contrastive) as a reference point. We further remove key ingredients one at a time: \textbf{CausalVAE branch} (bypassing the structural module), \textbf{state alignment loss} in Stage-2, \textbf{multi-step rollout supervision} (single-step only), and the \textbf{contrastive objective}. Finally, we test \textbf{alpha-gated fusion} by turning the gate off, and report sensitivity to \textbf{stage split schedules} and \textbf{rollout policies} (curriculum vs.\ mixed).

\begin{table}[t]
\centering
\tiny
\setlength{\tabcolsep}{3pt}
\renewcommand{\arraystretch}{0.92}
\caption{Ablations of training strategy, module removal, and hyperparameter sensitivity in the Physics benchmark.}
\label{tab:ablation}
\resizebox{0.78\linewidth}{!}{%
\begin{tabular}{@{}l
S[table-format=2.2]
S[table-format=2.2]
S[table-format=2.2]
S[table-format=2.2]@{}}
\toprule
Variant & \multicolumn{1}{c}{H@1 $\uparrow$} & \multicolumn{1}{c}{MRR $\uparrow$} & \multicolumn{1}{c}{CF-H@1 $\uparrow$} & \multicolumn{1}{c}{CF-MRR $\uparrow$} \\
\midrule
\textbf{three-stage training (ours)} & {\bfseries 91.00} & {\bfseries 95.16} & 31.00 & 52.15 \\
joint training (w/o stage split) & 90.33 & 94.75 & 28.00 & 51.05 \\
Baseline (GNN\_Contrastive) & 88.00 & 93.37 & 11.50 & {\bfseries 52.96} \\
\midrule
w/o CausalVAE branch & 80.00 & 89.55 & 27.67 & 50.06 \\
w/o state align loss & 9.00 & 28.02 & 26.67 & 50.57 \\
single-step rollout & 90.67 & 94.91 & 30.67 & 52.31 \\
w/o contrastive loss & 89.33 & 94.36 & 29.67 & 51.79 \\
\midrule
gate off & 1.33 & 5.71 & 13.67 & 28.09 \\
stage split: s1\_8\_s2\_40 & {\bfseries 91.00} & 95.08 & 28.67 & 51.45 \\
stage split: s1\_12\_s2\_48 & 90.67 & 94.90 & 27.33 & 50.71 \\
rollout policy: curriculum & {\bfseries 91.00} & 95.06 & {\bfseries 31.33} & 52.32 \\
rollout policy: mixed & 90.00 & 94.75 & 29.67 & 51.88 \\
\bottomrule
\end{tabular}
}
\end{table}

\vspace{0.3em}
\noindent\textit{Note.}
Unless otherwise specified, each ablation changes only one factor relative to \textbf{three-stage training (ours)}
(Ours setting: stage1=20 epochs, stage2=80 epochs, late-mixed rollout, contrastive loss, gate on).
Abbreviations: w/o = without; state align = state alignment loss; s1/s2 = stage1/stage2 epoch counts (e.g. s1\_8\_s2\_40 means stage1 8, stage2 40 epochs); CF = counterfactual (CF-H@1/CF-MRR defined in Sec.~\ref{subsec:metrics}).

Overall, three-stage optimization yields consistently better counterfactual retrieval than joint training while maintaining factual accuracy. Notably, removing the alignment objective or disabling the fusion gate can severely degrade factual retrieval, indicating that anchoring and controlled fusion are critical for stabilizing the learned causal latent space in sequential rollouts.

\section{Discussion}
The plug-in causal layer improves counterfactual reliability while keeping factual retrieval competitive across matched backbones, suggesting that gains come from structural constraints rather than architecture replacement.
From a broader perspective, this supports a modular path for robust world modeling: causal structure can be added to existing predictive systems without redesigning the full backbone.
Current limitations remain in strong nonlinear regimes and long-horizon drift; extending this approach to larger JEPA-style predictive architectures is a natural next step~\cite{assran2023ijepa,bardes2024revisiting,nam2026causaljepa}.

\section{Conclusion}
We presented a plug-in CausalVAE enhancement for latent world models that improves intervention-aware counterfactual behavior while keeping factual performance competitive.
Across diverse backbones and benchmarks, the results show that explicit latent causal factorization can deliver practical gains beyond standard predictive training, especially on intervention-focused metrics.
The observed first-order structural alignment further suggests that predictive latent spaces can be made more interpretable without sacrificing core forecasting ability.

\bibliographystyle{plainnat}
\clearpage
\bibliography{main}
\clearpage
\appendix
\section{Experimental Protocol Details}
\label{app:protocol_details}

\paragraph{Fair comparison setup.}
Across all baselines and domains, we enforce paired, matched runs: each Baseline model and its Baseline+CausalVAE counterpart use the same train/validation/test split, the same seed index, and the same data ordering. When official splits are provided by a benchmark, we use them directly; otherwise we create one fixed split per domain and reuse it for all methods. For every paired comparison, we keep optimizer family, learning-rate schedule, number of updates/epochs, and rollout horizon settings identical, so the only intended change is whether the causal structural branch is enabled.

\paragraph{Concrete defaults used in our implementation.}
Tab.~\ref{tab:train_defaults} lists the concrete defaults used in paired comparisons.

\begin{table}[htbp]
\centering
\caption{Concrete training defaults used for paired comparisons.}
\label{tab:train_defaults}
\resizebox{\linewidth}{!}{%
\begin{tabular}{@{}lll@{}}
\toprule
Setting & Default & Scope \\
\midrule
Baseline script & \texttt{train\_baselines.py} & all baseline / +CausalVAE paired runs \\
Seed & 42 & shared within each paired comparison \\
Optimizer & Adam & shared within each paired comparison \\
Batch size & 1024 & \texttt{train\_baselines.py} default \\
Learning rates & \texttt{slr=lr=transit-lr=finetune-lr}=5e-4 & \texttt{train\_baselines.py} default \\
Epoch settings & \texttt{epochs=100}, \texttt{pretrain-epochs=100}, \texttt{finetune-epochs=100} & \texttt{train\_baselines.py} default \\
Stage-3 script & \texttt{train\_stage3\_three\_suggestions.py} & transition refinement runs \\
Stage-3 batch size & 256 & Stage-3 default \\
Stage-3 learning rate & 1e-4 & Stage-3 default \\
Stage-3 epochs & 60 & Stage-3 default \\
Train rollout length & 5 (non-mixed), up to 10 (mixed / late-mixed) & rollout schedule configuration \\
\bottomrule
\end{tabular}%
}
\end{table}
\FloatBarrier

\paragraph{Baseline model configurations (with reference protocol).}
To make baseline comparisons fully reproducible, we explicitly document both (i) the reference baseline protocol from Ke \etal~\cite{ke2021systematic} and (ii) the exact defaults used in our codebase for AE/VAE/GNN/Modular baselines. As in the reference protocol, we keep model family, objective type, and rollout evaluation horizons aligned across methods, and only change the intended method component in paired comparisons. An itemized reference-vs-implementation summary is provided in Tab.~\ref{tab:baseline_configs}.

\begin{table}[htbp]
\centering
\caption{Baseline configuration details: reference protocol vs. our implementation defaults.}
\label{tab:baseline_configs}
\footnotesize
\begin{tabularx}{\linewidth}{@{}p{0.18\linewidth}p{0.38\linewidth}p{0.38\linewidth}@{}}
\toprule
Component & Ke \etal~\cite{ke2021systematic} reference setup & Our implementation defaults \\
\midrule
Baseline families & AE, VAE, GNN, Modular & AE, VAE, GNN, Modular (same family split) \\
Encoder/decoder backbone & Kipf-style CNN + MLP encoder/decoder; medium setting; object-wise latents for structured models~\cite{kipf2020cswm}. & \texttt{encoder=small}; shared world-model backbone in \texttt{train\_baselines.py}. \\
Latent representation & Fixed per-object embedding (reported as 32 per object in reference setup). & \texttt{embedding-dim-per-object}=5, \texttt{num-objects}=5 (domain-adjusted when needed). \\
Transition parameterization & AE/VAE: MLP transition; GNN: pairwise message passing; Modular: object-wise modular transition. & Model-type-specific transitions selected by \texttt{--vae}, \texttt{--gnn}, \texttt{--modular}. \\
Training objectives & NLL and contrastive settings; ranking metrics at 1/5/10-step horizons. & NLL (default) and contrastive route (\texttt{--contrastive}); factual metrics reported at $h\in\{1,5,10\}$. \\
NLL-type optimization budget & Adam, lr \(=5\times10^{-4}\), batch size 512, 100-epoch settings (reported). & \texttt{pretrain-epochs}=100, \texttt{epochs}=100, \texttt{finetune-epochs}=100. \\
Optimizer / LR / batch size & Adam, lr \(=5\times10^{-4}\), batch size 512 (reported). & Adam; \texttt{lr=slr=transit-lr=finetune-lr}=5e-4; \texttt{batch-size}=1024. \\
Action/state interface & Action-conditioned transition and object-centric intervention setting. & \texttt{action-dim}=5 with one-hot action encoding over object-action slots. \\
\bottomrule
\end{tabularx}
\end{table}
\FloatBarrier

\paragraph{Detailed baseline architectures and loss functions.}
We provide a detailed, implementation-level specification for AE/VAE/GNN/Modular baselines in this section, aligned with the reference baseline protocol in Ke \etal~\cite{ke2021systematic}. The alignment is at the level of model family split and objective settings (NLL, contrastive). The architecture-level mapping is summarized in Tab.~\ref{tab:baseline_arch_details}, and concrete optimization defaults are listed in Tab.~\ref{tab:baseline_loss_defaults}.

\paragraph{Shared notation.}
Given $(o_t,a_t,o_{t+1})$, each baseline maps observations to latent states $z_t=E_\theta(o_t)$ and $z_{t+1}=E_\theta(o_{t+1})$, predicts next latent $\hat{z}_{t+1}=F_\phi(z_t,a_t)$, and optionally reconstructs observations by $\hat{o}_t=D_\eta(z_t)$ and $\hat{o}_{t+1}=D_\eta(\hat{z}_{t+1})$.

\begin{table}[htbp]
\centering
\caption{Baseline backbone details used in this work (corresponding to the AE/VAE/GNN/Modular split in Ke \etal~\cite{ke2021systematic}).}
\label{tab:baseline_arch_details}
\small
\begin{tabularx}{\linewidth}{@{}p{0.11\linewidth}p{0.23\linewidth}p{0.27\linewidth}p{0.33\linewidth}@{}}
\toprule
Model & Latent structure & Transition module & Key implementation defaults \\
\midrule
AE & Monolithic latent vector. & MLP transition over latent + action. & \texttt{hidden-dim}=512, \texttt{embedding-dim-per-object}=5, \texttt{num-objects}=5. \\
VAE & Monolithic latent with posterior $(\mu,\log\sigma^2)$. & Same as AE for transition; stochastic latent sampling. & \texttt{--vae} flag; same transition width and action interface as AE. \\
GNN & Object-wise factored latents. & Graph message passing transition (pairwise interaction bias). & \texttt{--gnn} flag; action-conditioned transition with optional \texttt{--ignore-action}/\texttt{--copy-action}. \\
Modular & Object-wise factored latents. & Object-wise modular transition (higher-order interaction capacity). & \texttt{--modular} flag; per-object embedding and shared action interface. \\
\bottomrule
\end{tabularx}
\end{table}

\paragraph{Baseline objectives in our notation.}
Following the baseline definitions in Ke \etal~\cite{ke2021systematic}, we use the following loss settings with symbols matched to this paper:
\begin{align}
\mathcal{L}_{\mathrm{NLL}} &= \ell_{\mathrm{rec}}(o_t,\hat{o}_t)
+ \ell_{\mathrm{dyn}}(z_{t+1},\hat{z}_{t+1}).
\end{align}
In Ke \etal~\cite{ke2021systematic}, reconstruction terms are written as BCE and transition terms as MSE. In our implementation, reconstruction and transition are optimized with MSE-style regression losses.

\paragraph{Contrastive-route objective.}
For decoder-free training, we optimize encoder and transition jointly:
\begin{align}
\mathcal{L}_{\mathrm{con}} &= d\!\left(\hat{z}_{t+1}, z_{t+1}\right)
+ \max\!\left(0,\ \gamma - d\!\left(\tilde{z}_{t+1}, z_{t+1}\right)\right),
\end{align}
where $\tilde{z}_{t+1}$ is a negative latent sampled by batch shuffling and $d(\cdot,\cdot)$ is Euclidean/MSE distance in latent space. This matches the contrastive structure used in the reference protocol (positive transition matching plus hinge-separated negatives).

\begin{table}[htbp]
\centering
\caption{Baseline loss/optimization defaults in our implementation.}
\label{tab:baseline_loss_defaults}
\resizebox{\linewidth}{!}{%
\scriptsize
\begin{tabular}{@{}lll@{}}
\toprule
Item & Default value & Where used \\
\midrule
Optimizer & Adam & all baseline runs \\
Base learning rates & \texttt{slr=lr=transit-lr=finetune-lr}=5e-4 & baseline NLL and contrastive routes \\
Batch size & \texttt{1024} & \texttt{train\_baselines.py} \\
Epoch setting A & \texttt{100} & baseline NLL-type training budget \\
Epoch setting C & \texttt{100} & implementation default retained for reproducibility \\
Contrastive hinge & \texttt{1.0} (\texttt{--hinge}) & decoder-free contrastive route \\
Energy scale & \texttt{0.5} (\texttt{--sigma}) & contrastive route \\
Random seed & \texttt{42} & paired baseline vs.+CausalVAE comparison \\
\bottomrule
\end{tabular}%
}
\end{table}
\FloatBarrier

\paragraph{Three-stage schedule for +CausalVAE (not pure baselines).}
Our S1/S2/S3 schedule is used only for \emph{Baseline+CausalVAE} models: S1 optimizes factual reconstruction/prediction for the encoder-transition backbone; S2 freezes the backbone and optimizes the CausalVAE structural objective (reconstruction, KL, DAG, mask, and alignment when supervision exists); S3 freezes encoder/causal branch and fine-tunes transition dynamics for counterfactual stability using late-mixed rollouts. Pure baselines (without CausalVAE) are reported with baseline objective settings (NLL, contrastive) rather than this causal-branch schedule.

\paragraph{Intervention construction.}
Given a factual prefix $\mathbf{x}_{0:t_0}$ (and state prefix $\mathbf{s}_{0:t_0}$ when available), we apply interventions on selected targets:
\begin{equation}
do\!\left(s_i \leftarrow s_i + \Delta\right) \quad \text{or} \quad do\!\left(v_i \leftarrow v_i + \Delta\right).
\end{equation}
We evaluate both single-target and multi-target interventions. The intervention time $t_0$ is sampled from a valid prefix window to ensure sufficient history and future context. Intervention magnitudes are evaluated by small/medium/large sweeps (domain-normalized within each benchmark). Ground-truth counterfactual futures are obtained by re-simulating from the intervened state whenever simulator/state access is available.

\paragraph{Counterfactual benchmark construction details.}
We define one \emph{counterfactual query group} as $(\mathbf{x}_{0:t_0},\ \text{intervention spec},\ \mathbf{x}^{cf}_{t_0+1:t_0+H},\ \text{candidate set})$. The intervention spec includes target index (object/variable), axis or factor id when applicable, and magnitude $\Delta$. In default evaluator settings, we use up to \textbf{2000 query groups} per run (evaluator argument \texttt{max-samples=2000}). Each query uses one ground-truth counterfactual future. Negatives are sampled from non-matching episodes under the same benchmark protocol.

\begin{table}[htbp]
\centering
\caption{Detailed construction of counterfactual benchmark queries.}
\label{tab:cf_benchmark_details}
\small
\begin{tabularx}{\linewidth}{@{}p{0.19\linewidth}p{0.17\linewidth}p{0.23\linewidth}p{0.35\linewidth}@{}}
\toprule
Benchmark & Query groups (default) & Intervention variables & Notes \\
\midrule
Physics (3-body) & Up to 2000 per run. & Object id $i$, axis (x/y), magnitude $\Delta$. & Evaluator default includes \texttt{obj-idx}=0, \texttt{axis}=x, \texttt{delta}=3.0; axis/magnitude sweeps follow protocol settings. \\
2D Shapes / 3D Cubes & Up to 2000 per run. & Object-level pose/push-related factors. & Single- and multi-target variants; intervention index maps to object-slot metadata. \\
Chemistry & Up to 2000 per run. & Mechanism/factor id and magnitude $\Delta$. & Mechanism-level intervention metadata when available; otherwise use global structural protocol. \\
\bottomrule
\end{tabularx}
\end{table}

\paragraph{Retrieval protocol.}
For each query, we construct a candidate set $\{\mathbf{x}^{(j)}_{t_0+1:t_0+H}\}_{j=1}^{N}$ with one ground-truth future and $N-1$ negatives sampled from non-matching episodes under the same benchmark protocol. Candidates are scored by model likelihood/similarity. We report factual retrieval with H@1 and MRR at horizons $h\in\{1,5,10\}$ (evaluated separately per horizon), and counterfactual retrieval with CF-H@1 and CF-MRR on the counterfactual step under the same candidate-set construction. The full protocol choices are summarized in Tab.~\ref{tab:protocol}; detailed counterfactual query construction is given in Tab.~\ref{tab:cf_benchmark_details}.

\begin{table}[htbp]
\centering
\caption{Evaluation protocol summary.}
\label{tab:protocol}
\resizebox{\linewidth}{!}{%
\begin{tabular}{@{}llll@{}}
\toprule
Component & Choice & Variants & Notes \\
\midrule
Intervention time & $t_0$ & sampled over prefix window & consistent across methods \\
Targets & object / variable & single, multi & physics: pos/vel; pushing: pose/push \\
Magnitude & $\Delta$ & small/med/large & robustness check under intervention strength \\
Rollout horizon & $H$ & short/long & drift stress-test at long $H$ \\
Retrieval set & $N$, $h$ & fixed per benchmark & negatives from other episodes; $h\in\{1,5,10\}$ for factual metrics \\
Scoring & NLL / similarity & model-dependent & paired comparison protocol \\
\bottomrule
\end{tabular}
}
\end{table}
\FloatBarrier

\section{Reasonableness of Identifiability Assumptions}
\label{app:ident_assumptions}

This section explains why assumptions (I1)--(I5) used for the identifiability analysis in the main paper are reasonable for our training setup.

\noindent\textbf{(I1) Stage-2 decoupling.}
This is enforced by design: in Stage-2 we freeze the world-model backbone (encoder/transition) and optimize only the causal branch parameters. Hence the structural objective is the active objective for $(\psi,A,g_{\mathrm{align}})$.

\noindent\textbf{(I2) Latent invertibility.}
We treat invertibility as a local regularity assumption: in neighborhoods used for training/evaluation, encoder latents preserve sufficient information about the underlying state manifold. This is standard in identifiability analyses and is empirically supported when factual retrieval remains strong.

\noindent\textbf{(I3) Realizability + DAG.}
The realizability part is the usual population-level assumption that model class and optimization can represent the target mechanism. The acyclicity characterization is enforced through the differentiable DAG penalty in the main paper, which rules out cyclic solutions in the structural layer.

\noindent\textbf{(I4) Alignment anchoring + scale normalization.}
This is justified by our supervision protocol: when simulator states are available, the alignment head constrains latent coordinates to physically grounded targets, and normalization fixes scale degrees of freedom. Together they reduce permutation/scale ambiguity to a fixed coordinate frame.

\noindent\textbf{(I5) Population/global optimum.}
This is a theoretical idealization used for identifiability statements. In finite-sample SGD training, exact global optimality is not guaranteed; therefore the theorem should be interpreted as explaining the target solution structure rather than guaranteeing exact recovery in every run.

\noindent\textbf{Scope and boundary.}
The theorem claims identifiability in the anchored coordinate system induced by alignment, not under arbitrary unconstrained reparameterizations. When alignment is weak/noisy, dynamics are strongly non-stationary, or optimization is far from optimum, adjacency recovery may be only approximate.

\section{Supervision Sources and Objective Slot Selection}
\label{app:supervision_slots}

This section specifies, per benchmark family, (i) how Stage-2 supervision targets are obtained and (ii) how objective slots (the latent slots receiving intervention/alignment objectives) are determined.

\noindent\textbf{Physics (3-body gravitation).}
When simulator states are available, supervision is derived from physical state variables (\eg, position/velocity coordinates). Objective slots are assigned according to intervened object identity from simulator metadata; slot-level losses are applied to the corresponding latent slots, while non-target slots are treated as context.

\noindent\textbf{2D Shapes / 3D Cubes (block pushing).}
Supervision uses object-level annotations available from simulator or logged trajectory states (\eg, pose/push-related variables). Objective slots follow object-level intervention indices; if multiple objects are intervened, all corresponding slots are supervised jointly.

\noindent\textbf{Chemistry.}
When mechanism/factor labels are available, they are used as weak supervision targets for corresponding latent factors. Objective slots are mapped from mechanism-level intervention metadata when provided; otherwise, only globally defined structural objectives are applied.

\noindent\textbf{General rule used in this work.}
If benchmark metadata provides a direct intervention target index, we map it to the corresponding latent slot(s) and apply slot-aware objectives; if such metadata is unavailable, we do not enforce hard slot supervision and optimize only globally defined structural terms.

\section{Proof of the Identifiability Theorem (Main Text)}
\label{app:ident_proof}

\begin{proof}
By (I1), Stage-2 optimizes only the structural objective, hence identifiability of $A$ is determined solely by the structural learner. By (I2), $z_t=g(v_t)$ with invertible $g$, so optimizing over $(C_\psi,g_{\mathrm{align}})$ on inputs $z_t$ is equivalent (via reparameterization) to optimizing over $(C_\psi\!\circ g, g_{\mathrm{align}})$ on inputs $v_t$; thus we may work in the true causal coordinate domain without loss of generality. Let $(\hat\psi,\hat A,\hat g)$ and $(\psi',A',g')$ be two population-risk global minimizers, producing $\hat{\tilde z}_t=C_{\hat\psi}(z_t)$ and $\tilde z'_t=C_{\psi'}(z_t)$. By (I4), both minimizers satisfy $\hat g(\hat{\tilde z}_t)=s_t$ a.s.\ and $g'(\tilde z'_t)=s_t$ a.s. Define the (a.s.) invertible change of coordinates $T$ by $\tilde z'_t=T(\hat{\tilde z}_t)$. Then $\hat g(T(\hat{\tilde z}_t))=s_t=\hat g(\hat{\tilde z}_t)$ a.s.\ under the same normalization, so by the axis-fixing part of (I4) we have $T=I$ and hence $\tilde z'_t=\hat{\tilde z}_t$ a.s. Therefore all global minima share the same anchored latent $\tilde{z}_t$, and by (I3) the structural-consistency terms $(\mathcal{L}_{\mathrm{rec}},\mathcal{L}_{\mathrm{Mask}},\mathcal{L}_{\mathrm{KL}})$ together with acyclicity enforce that the minimizing adjacency is unique in this coordinate system, yielding $A'=\hat A=A^\star$.
\end{proof}

\FloatBarrier
\end{document}